\documentclass{article} 
\usepackage{nips12submit_e,times}

\usepackage{amsmath,amssymb,amsthm,amsfonts,mathtools,color,nicefrac,graphicx,multirow,bbm,rotating,slashbox}

\newcommand{\ignore}[1]{}

\newtheorem{theorem}{Theorem}

\newtheorem{lemma}{Lemma}
\newtheorem{proposition}{Proposition}

\theoremstyle{definition}

\theoremstyle{remark}

\makeatletter
\newtheorem*{rep@theorem}{\rep@title}
\newcommand{\newreptheorem}[2]{%
\newenvironment{rep#1}[1]{%
 \def\rep@title{#2 \ref{##1}}%
 \begin{rep@theorem}}%
 {\end{rep@theorem}}}
\makeatother

\newreptheorem{theorem}{Theorem}
\newreptheorem{lemma}{Lemma}
\newreptheorem{corollary}{Corollary}

\newcommand\bul{\scriptscriptstyle \bullet}

\newcommand{\norm}[1]{\left\lVert{#1}\right\rVert}

\newcommand{\Ep}[2]{\mathbb{E}_{#1}\left[{#2}\right]}
\renewcommand{\O}[1]{\mathcal{O}\left({#1}\right)}
\newcommand{\Th}[1]{{\Theta}\left({#1}\right)}

\newcommand\R{\mathbb{R}}

\newcommand\N{\mathbb{N}}

\newcommand\del{\delta}
\newcommand\eps{\epsilon}
\newcommand\gam{\gamma}
\newcommand\lam{\lambda}
\newcommand\sig{\sigma}
\DeclareMathOperator{\diag}{diag}
\DeclareMathOperator{\trace}{trace}

\renewcommand{\dh}[1]{{#1}^{\nicefrac{1}{2}}}
\newcommand{\dhi}[1]{{#1}^{-\nicefrac{1}{2}}}

\newcommand{\empRad}[2]{\widehat{\mathcal{R}}_{{#1}}\left({#2}\right)}

 \newcommand\w{\mathbf{w}}
 \renewcommand\r{\mathbf{r}}
 \renewcommand\c{\mathbf{c}}
 \newcommand\Rset{\mathcal{R}}
 \newcommand\Cset{\mathcal{C}}
 
 \newcommand\Del[1]{\Delta_{[{#1}]}}
  \newcommand\RCnorm[1]{\left\|{#1}\right\|_{(\Rset,\Cset)}}
  \newcommand\p{\mathbf{p}}
  \newcommand\tp{\widetilde{\p}}
  \newcommand\hp{\widehat{\p}}
  \newcommand\prow{\p_{\mathrm{row}}}
  \newcommand\pcol{\p_{\mathrm{col}}}
  \newcommand\tprow{\tp_{\mathrm{row}}}
  \newcommand\tpcol{\tp_{\mathrm{col}}}

  \newcommand\prowi{\p_{i\bul}}
  \newcommand\pcolj{\p_{\bul j}}

  \newcommand\tprowi{\tp_{i\bul}}
  
  \newcommand\hprowi{\hp_{i\bul}}
  \newcommand\hpcolj{\hp_{\bul j}}
  \newcommand\pij{\p_{ij}}
  \newcommand\wtrnorm[2]{\left\|{#2}\right\|_{\mathrm{tr}({#1})}}
  \newcommand\trnorm[1]{\left\|{#1}\right\|_{\mathrm{tr}}}
  \newcommand\frnorm[1]{\left\|{#1}\right\|_{\mathrm{F}}}
  \newcommand\rcnorm[1]{\wtrnorm{\r,\c}{{#1}}}

  \newcommand\tnorm[1]{\left\|{#1}\right\|_2}
  \newcommand\ttnorm[1]{\left\|{#1}\right\|_2^2}
  \newcommand\maxnorm[1]{\left\|{#1}\right\|_{\mathrm{max}}}
  \newcommand\conv[1]{\mathrm{Conv}\left\{{#1}\right\}}
  \newcommand\Rnorm[1]{\left\|{#1}\right\|_{\Rset}}
  \newcommand\Cnorm[1]{\left\|{#1}\right\|_{\Cset}}

\newcommand\TT{\mathrm{T}}
\newcommand\MM{\mathrm{M}}

\newcommand\xab{\mathcal{X}_{A,B}}

\title{Matrix reconstruction with the local max norm}

\author{
Rina Foygel\\
Department of Statistics\\
Stanford University\\
\texttt{rinafb@stanford.edu} \\
\And
Nathan Srebro\\
Toyota Technological Institute at Chicago\\
\texttt{nati@ttic.edu} \\
\And
Ruslan Salakhutdinov\\
Dept.\ of Statistics and Dept.\ of Computer Science
University of Toronto\\
\texttt{rsalakhu@utstat.toronto.edu} \\
}

\nipsfinalcopy 

\begin{document}

\maketitle

\begin{abstract}
We introduce a new family of matrix norms, the ``local max'' norms, generalizing existing methods such as the max norm, the trace norm (nuclear norm), and the weighted or smoothed weighted trace norms, which have been extensively used in the literature as regularizers for matrix reconstruction problems. We show that this new family can be used to interpolate between the (weighted or unweighted) trace norm and the more conservative max norm. We test this interpolation on simulated data and on the large-scale Netflix and MovieLens ratings data, and find improved accuracy relative to the existing matrix norms. We also provide theoretical results showing learning guarantees for some of the new norms.
\end{abstract}

\section{Introduction}

In the matrix reconstruction problem, we are given a matrix $Y\in\R^{n\times m}$ whose entries are only partly observed, and would like to reconstruct the unobserved entries as accurately as possible. Matrix reconstruction arises in many modern applications, including the areas of collaborative filtering (e.g.\ the Netflix prize), image and video data, and others. This problem has often been approached using regularization with matrix norms that promote low-rank or approximately-low-rank solutions, including the trace norm (also known as the nuclear norm) and the max norm, as well as several adaptations of the trace norm described below.

In this paper, we introduce a unifying family of norms that generalizes these existing matrix norms, and that can be used to interpolate between the trace and max norms. We show that this family includes new norms, lying strictly between the trace and max norms, that give 
empirical and theoretical improvements over the existing norms. We give results allowing for large-scale optimization with norms from the new family.
Some proofs are deferred to the Supplementary Materials.

\paragraph{Notation} Without loss of generality we take $n\geq m$.  We
let $\R_+$ denote the nonnegative real numbers.  For any $n\in\N$, let
$[n]=\{1,\dots,n\}$, and define the simplex on $[n]$ as
$\Del{n}=\left\{\r\in\R^n_+:\sum_i \r_i=1\right\}$.  We analyze
situations where the locations of observed entries are sampled
i.i.d.~according to some distribution $\p$ on $[n]\times[m]$. We write
$\prowi=\sum_j \pij$ to denote the marginal probability of row $i$,
and $\prow=(\p_{1\bul},\dots,\p_{n\bul})\in\Del{n}$ to denote the
marginal row distribution. We define $\pcolj$ and $\pcol$ similarly
for the columns.

\subsection{Trace norm and max norm}

A common regularizer used in matrix reconstruction, and other matrix
problems, is the trace norm $\trnorm{X}$, equal to the sum of the
singular values of $X$. This norm can also be defined via a
factorization of $X$ \cite{ShraibmanSrebro}:
\begin{equation}
\label{eq:TraceFactors}\frac{1}{\sqrt{nm}}\trnorm{X}=\frac{1}{2}\min_{AB^{\top}=X}\left(\frac{1}{n}\sum_i
  \norm{A_{(i)}}^2+\frac{1}{m}\sum_j\norm{B_{(j)}}^2\right)\;,
\end{equation}
where $M_{(i)}$ denotes the $i$th row of a matrix $M$, and where the
minimum is taken over factorizations of $X$ of arbitrary
dimension---that is, the number of columns in $A$ and $B$ is
unbounded.  Note that we choose to scale the trace norm by
$1/\sqrt{nm}$ in order to emphasize that we are averaging the squared
row norms of $A$ and $B$.

Regularization with the trace norm gives good theoretical and
empirical results, as long as the locations of observed entries are
sampled uniformly (i.e.~when $\p$ is the uniform distribution on
$[n]\times[m]$), and, under this assumption, can also be used to
guarantee approximate recovery of an underlying low-rank matrix
\cite{ShraibmanSrebro,KMO,NW,RinaNatiCOLT}.

The factorized definition of the trace norm \eqref{eq:TraceFactors} allows for an intuitive
 comparison with the max norm, defined as \cite{ShraibmanSrebro}:
\begin{equation}\label{eq:MaxFactors}\maxnorm{X}=\frac{1}{2}\min_{AB^{\top}=X}\left(\sup_i \ttnorm{A_{(i)}}+\sup_j\ttnorm{B_{(j)}}\right)\;.\end{equation}
We see that the max norm measures the largest row norms in the factorization, while the rescaled trace norm instead considers the average row norms. The max norm is therefore an upper bound on the rescaled trace norm, and can be viewed as a more conservative regularizer. For the more general setting where $\p$ may not be uniform, Foygel and Srebro \cite{RinaNatiCOLT} show that the max norm is still an effective regularizer (in particular, bounds on error for the max norm are not affected by $\p$). On the other hand, Salakhutdinov and Srebro \cite{RusAndNatiNIPS} show that the trace norm is not robust to non-uniform sampling---regularizing with the trace norm may yield large error due to over-fitting on the rows and columns with high marginals. They obtain improved empirical results by placing more penalization on these over-represented rows and columns, described next.

\subsection{The weighted trace norm}

To reduce overfitting on the rows and columns with high marginal probabilities under the distribution $\p$,  Salakhutdinov and Srebro propose regularizing with the $\p$-weighted trace norm, 
\[\wtrnorm{\p}{X}:=\trnorm{\diag(\prow)^{\nicefrac{1}{2}}\cdot X\cdot
  \diag(\pcol)^{\nicefrac{1}{2}}}\;.\] 

If the row and the column of entries to be observed are sampled
independently (i.e.~$\p = \prow \cdot \pcol$ is a product
distribution), then the $\p$-weighted trace norm can be used to obtain
good learning guarantees even when $\prow$ and $\pcol$ are non-uniform
\cite{NW,FSSS}.  However, for non-uniform non-product sampling
distributions, even the $\p$-weighted trace norm can yield poor
generalization performance.  To correct for this, Foygel et al.\
\cite{FSSS} suggest adding in some ``smoothing'' to avoid
under-penalizing the rows and columns with low marginal probabilities,
and obtain improved empirical and theoretical results for matrix
reconstruction using the smoothed weighted trace norm:
\[\wtrnorm{\tp}{X}:=\trnorm{\diag(\tprow)^{\nicefrac{1}{2}}\cdot X\cdot \diag(\tpcol)^{\nicefrac{1}{2}}}\;,\]
where $\tprow$ and $\tpcol$ denote smoothed row and column marginals, given by
\begin{equation}\label{eq:alphaSmoothing}\tprow=(1-\zeta)\cdot\prow + \zeta\cdot\nicefrac{1}{n}\text{ and }\tpcol=(1-\zeta)\cdot\pcol + \zeta\cdot\nicefrac{1}{m}\;,\end{equation}
for some choice of smoothing parameter $\zeta$ which may be selected with cross-validation\footnote{Our $\zeta$ parameter here is equivalent to $1-\alpha$ in \cite{FSSS}.}. The smoothed empirically-weighted trace norm is also studied in \cite{FSSS}, where $\prowi$ is replaced with $\hprowi=\frac{\text{\# observations in row $i$}}{\text{total \# observations}}$, the empirical marginal probability of row $i$ (and same for $\hpcolj$). Using empirical rather than ``true'' weights yielded lower error in experiments in \cite{FSSS}, even when the true sampling distribution was uniform.

More generally, for any weight vectors $\r\in\Del{n}$ and $\c\in\Del{m}$ and a matrix $X\in\R^{n\times m}$, the $(\r,\c)$-weighted trace norm is given by
\[\wtrnorm{\r,\c}{X}=\trnorm{\dh{\diag(\r)}\cdot X\cdot\dh{\diag(\c)}}\;.\]
Of course, we can easily obtain the existing methods of the uniform trace norm, (empirically) weighted trace norm, and smoothed (empirically) weighted trace norm as special cases of this formulation. Furthermore, the max norm is equal to a supremum over all possible weightings \cite{maxnorm}:
\[\maxnorm{X}=\sup_{\r\in\Del{n},\c\in\Del{m}}\rcnorm{X}\;.\]

\section{The local max norm}

We consider a generalization of these norms, which lies ``in between'' the trace norm and max norm.  For any $\Rset\subseteq \Del{n}$ and $\Cset\subseteq\Del{m}$, we define the $(\Rset,\Cset)$-norm of $X$:
\[\RCnorm{X}=\sup_{\r\in\Rset,\c\in\Cset}\rcnorm{X}\;.\]
This gives a norm on matrices, except in the trivial case where, for some $i$ or some $j$, $\r_i=0$ for all $\r\in\Rset$ or $\c_j=0$ for all $\c\in\Cset$.

We now show some existing and novel norms that can be obtained using local max norms.

\subsection{Trace norm and max norm} 

We can obtain the max norm by taking the largest possible $\Rset$ and $\Cset$, i.e.\
$\maxnorm{X}=\norm{X}_{({\Del{n},\Del{m}})}$,
and similarly we can obtain the $(\r,\c)$-weighted trace norm by
taking the singleton sets $\Rset=\{\r\}$ and $\Cset=\{\c\}$.  As
discussed above, this includes the standard trace norm (when $\r$
and $\c$ are uniform), as well as the weighted, empirically weighted,
and smoothed weighted trace norm.

\subsection{Arbitrary smoothing} 

When using the smoothed weighted max norm, we need to choose the
amount of smoothing to apply to the marginals, that is, we need to
choose $\zeta$ in our definition of the smoothed row and column
weights, as given in \eqref{eq:alphaSmoothing}. Alternately, we could
regularize simultaneously over all possible amounts of smoothing by
considering the local max norm with
\[\Rset=\left\{(1-\zeta)\cdot \prow + \zeta\cdot \nicefrac{1}{n}:\text{ any } \zeta\in[0,1]\right\}\;,\]
and same for $\Cset$.
That is, $\Rset$ and $\Cset$ are line segments in the simplex---they are
larger than any single point as for the uniform or weighted trace norm (or smoothed
weighted trace norm for a fixed amount of smoothing), but smaller than
the entire simplex as for the max norm.

\subsection{Connection to $(\beta,\tau)$-decomposability} 

Hazan et al.\ \cite{Hazan} introduce a class of matrices defined by a
property of $(\beta,\tau)$-decomposability: a matrix $X$ satisfies
this property if there exists a factorization $X=AB^{\top}$ (where $A$ and $B$ may have an arbitrary number of columns)
such that
\[\max\left\{\max_i \norm{A_{(i)}}^2_2,\max_j \norm{B_{(j)}}^2_2\right\}\leq 2\beta, \ \sum_i \norm{A_{(i)}}^2_2+\sum_j \norm{B_{(j)}}^2_2\leq\tau\;,\]
where $A_{(i)}$ and $B_{(j)}$ are the $i$th row of $A$ and the $j$th
row of $B$, respectively\footnote{Hazan et al.\ state the property
differently, but equivalently, in terms of a semidefinite matrix
decomposition.}.

Comparing with \eqref{eq:TraceFactors} and \eqref{eq:MaxFactors},
we see that the $\beta$ and $\tau$ parameters essentially correspond to the
max norm and trace norm, with the max norm being the minimal $2\beta^*$
such that the matrix is $(\beta^*,\tau)$-decomposable for some $\tau$,
and the trace norm being the minimal $\tau^*/2$ such that the matrix is
$(\beta,\tau^*)$-decomposable for some $\beta$.  However, Hazan et
al.~go beyond these two extremes, and rely on balancing both $\beta$
and $\tau$: they establish learning guarantees (in an adversarial
online model, and thus also under an arbitrary sampling distribution
$\p$) which scale with $\sqrt{\beta\cdot \tau}$. It may therefore be
useful to consider a penalty function of the form:
\begin{equation}\label{eq:BetaTauCombine}\mathrm{Penalty}_{(\beta,\tau)}(X)= \min_{X=AB^{\top}}\left\{\sqrt{\max_i \norm{A_{(i)}}^2_2+\max_j \norm{B_{(j)}}^2_2}\cdot \sqrt{\sum_i \norm{A_{(i)}}^2_2+\sum_j \norm{B_{(j)}}^2_2}\right\}\;.\end{equation}
(Note that $\max\left\{\max_i \norm{A_{(i)}}^2_2,\max_j \norm{B_{(j)}}^2_2\right\}$ is replaced with $\max_i \norm{A_{(i)}}^2_2+\max_j\norm{B_{(j)}}^2_2$, for later convenience. This affects the value of the penalty
function by at most a factor of $\sqrt{2}$.)

This penalty function does not appear to be convex in $X$. However,
the proposition below (proved in the Supplementary Materials)
 shows that we can use a (convex) local max norm
penalty to compute a solution to any objective function with a penalty
function of the form \eqref{eq:BetaTauCombine}:
\begin{proposition}\label{lem:BetaTau}
 Let $\widehat{X}$ be the minimizer of a penalized loss function with this modified penalty,
\[\widehat{X}:=\arg\min_X\left\{\mathrm{Loss}(X) +\lam \cdot  \mathrm{Penalty}_{(\beta,\tau)}(X)\right\}\;,\]
where $\lam\geq 0$ is some penalty parameter and $\mathrm{Loss}(\cdot)$ is any convex function. Then, for some penalty parameter $\mu\geq 0$ and some $t\in[0,1]$, 
\[\widehat{X}=\arg\min_X \left\{\mathrm{Loss}(X)+\mu\cdot\RCnorm{X}\right\}\;,\text{ where }\]
\[\Rset=\left\{\r\in\Del{n}:\r_i\geq \frac{t}{1+(n-1)t}\ \forall i\right\}\text{ and } \Cset=\left\{\c\in\Del{m}:\c_j\geq \frac{t}{1+(m-1)t}\ \forall j\right\}\;.\]
\end{proposition}
We note that $\mu$ and $t$ cannot be determined based on $\lambda$ alone---they will depend on the properties of the unknown solution $\widehat{X}$. 

Here the sets $\Rset$ and $\Cset$ impose a lower bound on each of the
weights, and this lower bound can be used to interpolate between the
max and trace norms: when $t=1$, each $\r_i$ is lower bounded by
$\nicefrac{1}{n}$ (and similarly for $\c_j$), i.e.\ $\Rset$ and $\Cset$
are singletons containing only the uniform weights and we obtain the
trace norm.  On the other hand, when $t=0$, the weights are lower-bounded 
by zero, and so any weight vector is allowed, i.e.\ $\Rset$ and
$\Cset$ are each the entire simplex and we obtain the max norm.
Intermediate values of $t$ interpolate between the trace norm and
max norm and correspond to different balances between $\beta$ and
$\tau$.

\subsection{Interpolating between trace norm and max norm} 
\label{sec:int}

We next turn to an interpolation which relies on an upper bound,
rather than a lower bound, on the weights. Consider
\begin{equation}\label{eq:Between}\Rset_{\eps}=\left\{\r\in\Del{n}:\r_i\leq \eps\ \forall i\right\}\text{ and }\Cset_{\del}=\left\{\c\in\Del{n}:\c_j\leq \del\ \forall j\right\}\;,\end{equation}
for some $\eps\in[\nicefrac{1}{n},1]$ and $\del\in[\nicefrac{1}{m},1]$. The $(\Rset_{\eps},\Cset_{\del})$-norm is then equal to the (rescaled) trace norm when we choose  $\eps=\nicefrac{1}{n}$ and $\del=\nicefrac{1}{m}$, and is equal to the max norm when we choose $\eps=\del=1$. Allowing $\eps$ and $\del$ to take intermediate values gives a smooth interpolation between these two familiar norms, and may be useful in situations where we want more flexibility in the type of regularization. 

We can generalize this to an interpolation between the max norm and a smoothed weighted trace norm, which we will use in our 
experimental results. We consider two generalizations---for each one, we state a definition of $\Rset$, with $\Cset$ defined analogously. The first is multiplicative:
\begin{equation}\label{eq:interpolate_mult}\Rset^{\times}_{\zeta,\gamma}:=\left\{\r\in\Del{n}:\r_i\leq \gamma\cdot \left((1-\zeta)\cdot\prowi+\zeta\cdot\nicefrac{1}{n}\right) \ \forall i\right\}\;,\end{equation}
where $\gamma=1$ corresponds to choosing the singleton set $\Rset^{\times}_{\zeta,\gamma}=\left\{(1-\zeta)\cdot\prow+\zeta\cdot\nicefrac{1}{n}\right\}$ (i.e.\ the smoothed weighted trace norm), while $\gamma=\infty$ corresponds to the max norm (for any choice of $\zeta$) since we would get $\Rset^{\times}_{\zeta,\gamma}=\Del{n}$.

The second option for an interpolation is instead defined with an  exponent:
\begin{equation}\label{eq:interpolate_exp}\Rset_{\zeta,\tau}:=\left\{\r\in\Del{n}:\r_i\leq  \left((1-\zeta)\cdot\prowi+\zeta\cdot\nicefrac{1}{n}\right)^{1-\tau} \ \forall i\right\}\;.\end{equation}
Here $\tau=0$ will yield the singleton set corresponding to the smoothed weighted trace norm, while $\tau=1$ will yield $\Rset_{\zeta,\tau}=\Del{n}$, i.e.\ the max norm, for any choice of $\zeta$.

We find the second (exponent) option to be more natural, because each of the row marginal bounds will reach $1$ simultaneously when $\tau=1$, and hence we use this version in our experiments. On the other hand, the multiplicative version is easier to work with theoretically, and we use this in our learning guarantee in Section \ref{sec:Learning}. If all of the row and column marginals satisfy some loose upper bound, then the two options will not be highly different.

\section{Optimization with the local max norm}\label{sec:Computation}

One appeal of both the trace norm and the max norm is that they are
both SDP representable \cite{FazelHindiBoyd,MMMF}, and thus easily
optimizable, at least in small scale problems.  Indeed, in the
Supplementary Materials we show that the local max norm is also SDP
representable, as long as the sets $\Rset$ and $\Cset$ can be written
in terms of linear or semi-definite constraints---this includes all
the examples we mention, where in all of them the sets $\Rset$ and
$\Cset$ are specified in terms of simple linear constraints.

However, for large scale problems, it is not practical to directly use
SDP optimization approaches.  Instead, and especially for very large
scale problems, an effective optimization approach for both the
trace norm and the max norm is to use the factorized versions of the
norms, given in \eqref{eq:TraceFactors} and \eqref{eq:MaxFactors}, and
to optimize the factorization directly (typically, only factorizations
of some truncated dimensionality are used)
\cite{fastMMMF,PMF,maxnorm}.
As we show in Theorem \ref{thm:Factors} below, a similar
factorization-optimization approach is also possible for any local max
norm with convex $\Rset$ and $\Cset$.  We further give a simplified
representation which is applicable when $\Rset$ and $\Cset$ are
specified through element-wise upper
bounds $R \in\R^n_+$ and $C \in\R^m_+$, respectively:
\begin{equation}\label{eq:Elementwise}\Rset=\{\r\in\Del{n}:\r_i\leq R_i\ \forall i\}\text{ and }\Cset=\{\c\in\Del{m}: \c_j\leq C_j\ \forall j\}\;,\end{equation}
with $0\leq R_i \leq
1$, $\sum_i R_i \geq 1$, $0 \leq C_j \leq 1$, $\sum_j C_j \geq 1$ to avoid triviality.
This includes the interpolation norms of Section \ref{sec:int}.

\begin{theorem}\label{thm:Factors}
If $\Rset$ and $\Cset$ are convex, then the $(\Rset,\Cset)$-norm can be calculated with the factorization
\begin{equation}\label{eq:GeneralFactors}\RCnorm{X}=\frac{1}{2}\inf_{AB^{\top}=X}\Big(\sup_{\r\in\Rset}\sum_i \r_i\ttnorm{A_{(i)}}+\sup_{\c\in\Cset}\sum_j \c_j \ttnorm{B_{(j)}}\Big)\;.\end{equation}
 In the special case when $\Rset$ and
$\Cset$ are defined by \eqref{eq:Elementwise}, writing $(x)_+:=\max\{0,x\}$, this simplifies to
\[\RCnorm{X}=\frac{1}{2}\inf_{{AB^{\top}=X;a,b\in\R}}\Big\{a+\sum_iR_i\left(\ttnorm{A_{(i)}}-a
\right)_+ + b + \sum_j C_j\left(\ttnorm{B_{(j)}}-b
\right)_+\Big\}\;.\]
\end{theorem}
\vspace{-.25in}
\begin{proof}[Proof sketch for Theorem \ref{thm:Factors}]
For convenience we will write $\dh{\r}$ to mean $\dh{\diag(\r)}$, and same for $\c$.  Using the trace norm factorization identity \eqref{eq:TraceFactors}, we have
\begin{multline*}
2\RCnorm{X}
=2\sup_{\r\in\Rset,\c\in\Cset}\trnorm{\dh{\r}\cdot X\cdot \dh{\c}}
=\sup_{\r\in\Rset,\c\in\Cset} \ \inf_{CD^{\top}=\dh{\r}\cdot X\cdot \dh{\c}}\left(\frnorm{C}^2+\frnorm{D}^2\right)\\
= \sup_{\r\in\Rset,\c\in\Cset}\inf_{AB^{\top}=X}\left(\frnorm{\dh{\r}\cdot A}^2+\frnorm{\dh{\c}\cdot B}^2\right)
\leq \inf_{AB^{\top}=X}\left(\sup_{\r\in\Rset}\frnorm{\dh{\r}A}^2+\sup_{\c\in\Cset}\frnorm{\dh{\c}B}^2\right)\;,
 \end{multline*}
where for the next-to-last step we set $C=\dh{\r}A$ and $D=\dh{\c}B$, and the last step follows because $\sup\inf\leq\inf\sup$ always (weak duality). The reverse inequality holds as well (strong duality), and is proved in the Supplementary Materials, where we also prove the special-case result.
\end{proof}

\section{An approximate convex hull and a learning guarantee}\label{sec:ConvexAndLearning}

In this section, we look for theoretical bounds on error for the problem of estimating unobserved entries in a matrix $Y$ that is approximately low-rank. Our results apply for either uniform or non-uniform sampling of entries from the matrix. 
We begin with a result comparing the $(\Rset,\Cset)$-norm unit ball to a convex hull of rank-$1$ matrices, which will be useful for proving our learning guarantee.

\subsection{Convex hull}\label{sec:Convex}

To gain a better theoretical understanding of the $(\Rset,\Cset)$ norm, we first need to define corresponding vector norms on $\R^n$ and $\R^m$. For any $u\in\R^n$, let
\[\Rnorm{u}:=\sqrt{\sup_{\r\in\Rset}\sum_i \r_i u_i^2} = \sup_{\r\in\Rset}\tnorm{\dh{\diag(\r)}\cdot u}\;.\]
We can think of this norm as a way to  interpolate between the $\ell_2$ and $\ell_{\infty}$ vector norms.
For example, if we choose $\Rset=\Rset_{\eps}$ as defined in \eqref{eq:Between}, then $\Rnorm{u}$ is equal to the root-mean-square of the $\eps^{-1}$ largest entries of $u$ whenever $\eps^{-1}$ is an integer.
Defining $\Cnorm{v}$ analogously for $v\in\R^m$, we can now relate these vector norms to the $(\Rset,\Cset)$-norm on matrices.
\begin{theorem}\label{thm:ConvexHull}
For any convex $\Rset\subseteq \Del{n}$ and $\Cset\subseteq\Del{m}$, the $(\Rset,\Cset)$-norm unit ball is bounded above and below by a convex hull as:
\[\conv{uv^{\top}\!\!:\!\Rnorm{u}=\Cnorm{v}=1}\subseteq \left\{X\!:\!\RCnorm{X}\leq1 \right\}\subseteq K_G\cdot \conv{uv^{\top}\!\!:\!\Rnorm{u}=\Cnorm{v}=1},\]
 where $K_G\leq 1.79$ is Grothendieck's constant, and implicitly $u\in\R^n$, $v\in\R^m$.
\end{theorem}

This result  is a nontrivial extension of Srebro and Shraibman \cite{ShraibmanSrebro}'s analysis for the max norm and the trace norm.
They show that the statement holds for the max norm, i.e.\ when $\Rset=\Del{n}$ and $\Cset=\Del{m}$, 
and that the trace norm unit ball is exactly equal to the corresponding convex hull (see Corollary 2 and Section 3.2 in their paper, respectively).\vspace{-.2cm}
\begin{proof}[Proof sketch for Theorem \ref{thm:ConvexHull}]To prove the first inclusion, given any $X=uv^{\top}$ with $\Rnorm{u}=\Cnorm{v}=1$, we apply the factorization result Theorem \ref{thm:Factors} to see that $\RCnorm{X}\leq 1$. Since the $(\Rset,\Cset)$-norm unit ball is convex, this is sufficient.
For the second inclusion, we state a weighted version of Grothendieck's Inequality (proof in the Supplementary Materials):
\begin{multline*}\sup\left\{\langle Y, UV^{\top}\rangle :U\in\R^{n\times k},V\in\R^{m\times k}, \tnorm{U_{(i)}}\leq a_i \ \forall i,\  \tnorm{V_{(j)}}\leq b_j\ \forall j\right\}\\
=K_G\cdot  \sup\left\{\langle Y,uv^{\top}\rangle : u\in\R^n,v\in\R^m, |u_i|\leq a_i \ \forall i,\  |v_j|\leq b_j\ \forall j\right\}\;.
\end{multline*}
We then apply this weighted inequality to the dual norm to the $(\Rset,\Cset)$-norm to prove the desired inclusion, as in Srebro and Shraibman \cite{ShraibmanSrebro}'s work for the max norm case (see Corollary 2 in their paper). Details are given in the Supplementary Materials.
\end{proof}

\subsection{Learning guarantee}\label{sec:Learning}

We now give our main matrix reconstruction result, which provides error bounds for  a family of norms interpolating between the max norm and the smoothed weighted trace norm.

\begin{theorem}\label{thm:Learning}
Let $\p$ be any distribution on $[n]\times[m]$. Suppose that, for some $\gam\geq 1$, 
$\Rset\supseteq \Rset^{\times}_{\nicefrac{1}{2},\gamma}\text{ and }\Cset\supseteq\Cset^{\times}_{\nicefrac{1}{2},\gamma}$,
 where these two sets are defined in \eqref{eq:interpolate_mult}.
Let $S=\{(i_t,j_t):t=1,\dots,s\}$ be a random sample of locations in the matrix drawn i.i.d.\ from $\p$, where $s\geq n$. Then, in expectation over the sample $S$,\vspace{-.2cm}
\[\sum_{ij}\pij \left|Y_{ij}-\widehat{X}_{ij}\right|\leq \underbrace{\inf_{\RCnorm{X}\leq \sqrt{k}}\sum_{ij}\pij \left|Y_{ij}-X_{ij}\right|}_{\text{Approximation error}} + \underbrace{\O{\sqrt{\frac{kn}{s}}}\cdot\left(1+\frac{\log(n)}{\sqrt{\gam}}\right)}_{\text{Excess error}}\;,\]
where $\widehat{X}=\arg\min_{\RCnorm{X}\leq \sqrt{k}}\sum_{t=1}^s \left|Y_{i_tj_t}-X_{i_tj_t}\right|$.
Additionally, if we assume that $s\geq n\log(n)$, then in the excess risk bound, we can reduce the term $\log(n)$ to $\sqrt{\log(n)}$.
\end{theorem}\vspace{-.2cm}
\begin{proof}[Proof sketch for Theorem \ref{thm:Learning}]
The main idea is to use the convex hull formulation from Theorem \ref{thm:ConvexHull} to show that, for any $X$ with $\RCnorm{X}\leq \sqrt{k}$, there exists a decomposition $X=X'+X''$ with
$\maxnorm{X'}\leq \mathcal{O}({\sqrt{k}})$ and $\wtrnorm{\tp}{X''}\leq \mathcal{O}({\sqrt{{k}/{\gamma}}})$, where $\tp$ represents the smoothed marginals with smoothing parameter $\zeta=\nicefrac{1}{2}$ as in \eqref{eq:alphaSmoothing}.
We then apply known bounds on the Rademacher complexity of the max norm unit ball \cite{ShraibmanSrebro} and the smoothed weighted trace norm unit ball \cite{FSSS}, to  bound  the Rademacher complexity of $\big\{X:\RCnorm{X}\leq\sqrt{k}\big\}$. This then yields a learning guarantee by Theorem 8 of Bartlett and Mendelson \cite{BartlettMendelson}. Details are given in the Supplementary Materials.
\end{proof}

As special cases of this theorem, we can re-derive the existing
results for the max norm and smoothed weighted trace norm.
Specifically, choosing $\gam=\infty$ gives us an excess error term of
order ${\sqrt{{kn}/{s}}}$ for the max norm, previously shown by
\cite{ShraibmanSrebro}, while setting $\gam=1$ yields an excess error
term of order ${\sqrt{{kn\log(n)}/{s}}}$ for the smoothed weighted
trace norm as long as $s\geq n\log(n)$, as
shown in \cite{FSSS}.

What advantage does this new result offer over the existing results for the max norm and for the smoothed weighted trace norm? To simplify the comparison, suppose we choose $\gam =\log^2(n)$, and define $\Rset= \Rset^{\times}_{\nicefrac{1}{2},\gam}$ and $\Cset= \Cset^{\times}_{\nicefrac{1}{2},\gam}$. Then, comparing to the max norm result (when $\gam=\infty$), we see that the excess error term is the same in both cases (up to a constant), but the approximation error term may in general be much lower for the local max norm than for the max norm.
Comparing next to the weighted trace norm (when $\gam=1$), we see that
the excess error term is lower by a factor of $\log(n)$ for the local
max norm. This may come at a cost of increasing the approximation
error, but in general this increase will be very small. In particular,
the local max norm result allows us to give a meaningful guarantee for
a sample size $s=\Th{kn}$, rather than requiring $s\geq
\Th{kn\log(n)}$ as for any trace norm result, but with a hypothesis
class significantly richer than the max norm constrained class (though
not as rich as the trace norm constrained class).

\vspace{-.1cm}

\section{Experiments}

We test the local max norm on simulated and real matrix reconstruction tasks, and compare its performance to the max norm, the uniform and empirically-weighted trace norms, and the smoothed empirically-weighted trace norm.

\vspace{-0.1in}
\subsection{Simulations}
\vspace{-0.1in}
We simulate $n\times n$ noisy matrices for $n=30,60,120,240$, where the underlying signal has rank $k=2$ or $k=4$, and we observe $s=3 k n$ entries (chosen uniformly without replacement). We performed $50$ trials for each of the $8$ combinations of $(n,k)$. 
\vspace{-0.1in}

\paragraph{Data} For each trial, we randomly draw a matrix $U\in\R^{n\times k}$ by drawing each row uniformly at random from the unit sphere in $\R^n$. We generate $V\in\R^{m\times k}$ similarly. We set $Y=UV^{\top}+\sig\cdot Z$, where the noise matrix $Z$ has i.i.d.\ standard normal entries and $\sig=0.3$ is a moderate noise level. We also divide the $n^2$ entries of the matrix into sets $S_0\sqcup S_1\sqcup S_2$ which consist of $s=3kn$ training entries, $s$ validation entries, and $n^2-2s$ test entries, respectively, chosen uniformly at random. 
\vspace{-0.1in}

\paragraph{Methods} 
We use the two-parameter family of norms defined in \eqref{eq:interpolate_exp}, but replacing the true marginals $\prowi$ and $\pcolj$ with the empirical marginals $\hprowi$ and $\hpcolj$. We consider $\zeta,\tau\in\{0,0.1,\dots,0.9,1\}$. For each $(\zeta,\tau)$ combination and each penalty parameter value $\lam\in\{2^1,2^2,\dots,2^{10}\}$, we compute the fitted matrix 
\begin{equation}\label{eq:alpha_tau_lambda}\widehat{X}=\arg\min\left\{\sum\!{}_{(i,j)\in S_0} \ (Y_{ij}-X_{ij})^2 + \lam\cdot \norm{X}_{(\Rset_{\zeta,\tau},\Cset_{\zeta,\tau})}\right\}\;.\end{equation}
(In fact, we use a rank-$8$ approximation to this optimization problem, as described in Section \ref{sec:Computation}.)
For each of the considered matrix norm methods, we use the validation set $S_1$ to select the best combination of $\zeta$, $\tau$, and $\lambda$, with restrictions on $\zeta$ and/or $\tau$ as specified by the definition of the method (see Table \ref{tab:MethodsForNorms}). We then report the error of the resulting fitted matrix  on the test set $S_2$.

\begin{table}[tdp]
\caption{Matrix fitting for the five methods used in experiments.}
\begin{center}
\begin{tabular}{|c|c|c|}
\hline
Norm&Fixed parameters&Free parameters\\\hline\hline
Max norm&$\zeta$ arbitrary; $\tau=1$&$\lam$\\\hline
(Uniform) trace norm&$\zeta=1$; $\tau=0$&$\lam$\\\hline
Empirically-weighted trace norm&$\zeta=0$; $\tau=0$&$\lam$\\\hline
Arbitrarily-smoothed emp.-wtd. trace norm&$\tau=0$&$\zeta$; $\lam$\\\hline
Local max norm&---&$\zeta$; $\tau$; $\lam$\\\hline
\end{tabular}
\end{center}
\label{tab:MethodsForNorms}
\vspace{-0.2in}
\end{table}%

\begin{figure}[tdp]
\hbox{ \centering \hspace{-0.2in}
\includegraphics[width=2.9in]{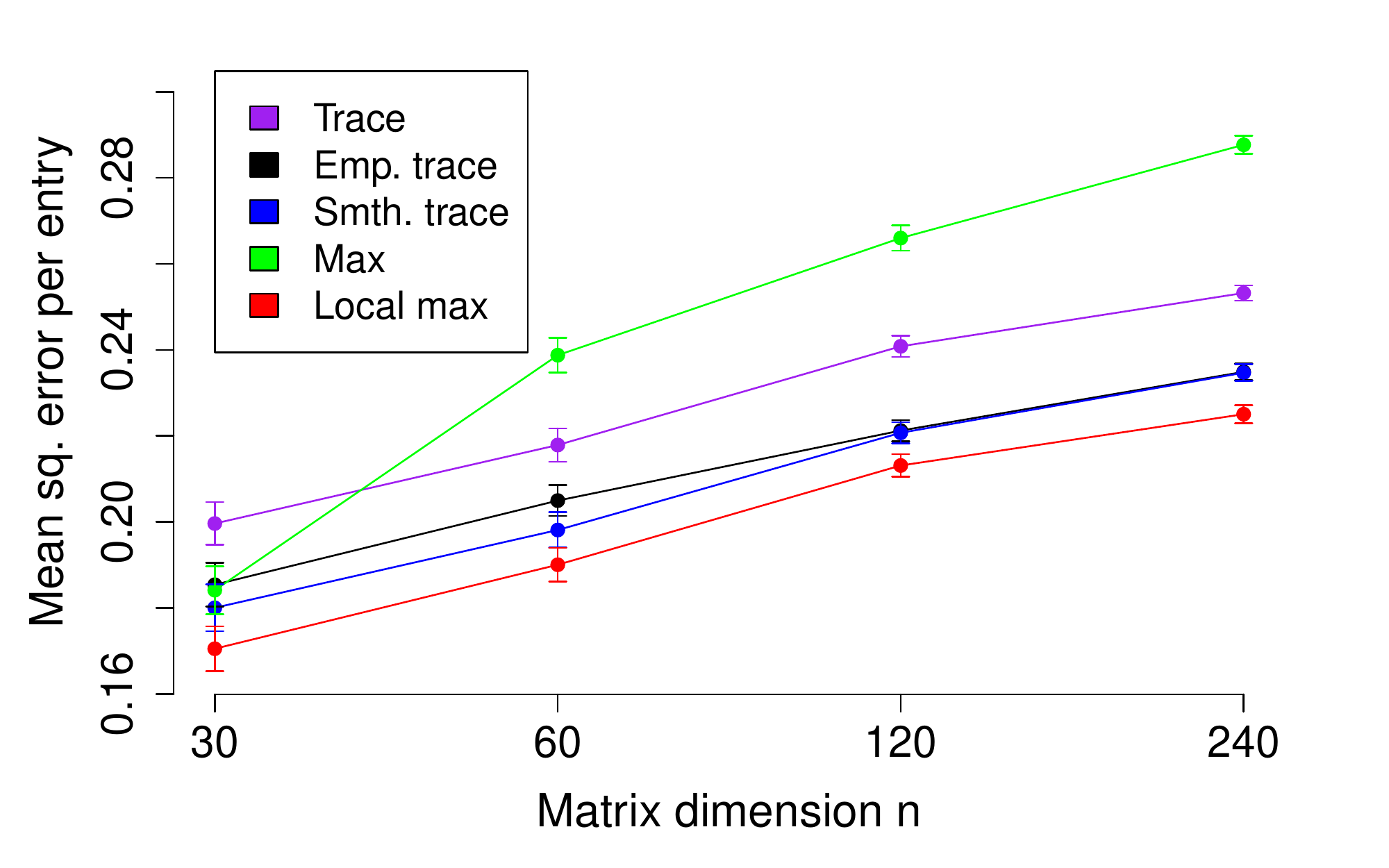}
\includegraphics[width=2.9in]{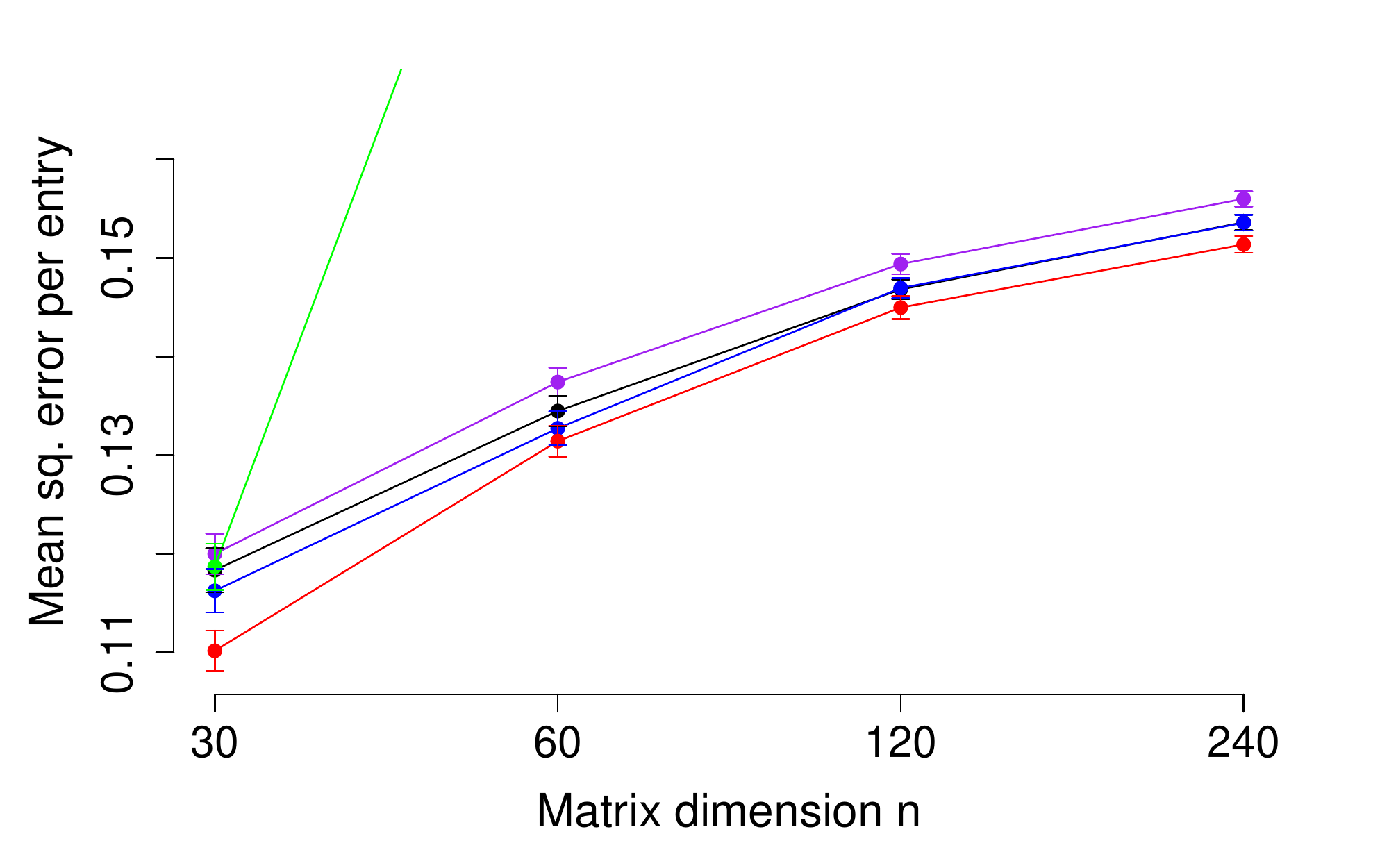}
}
\vspace{-0.2in}
\caption{\small Simulation results for matrix reconstruction with a rank-$2$ (left) or rank-$4$ (right) signal, corrupted by noise. The plot shows per-entry squared error averaged over $50$ trials, with standard error bars. For the rank-$4$ experiment,  max norm error exceeded $0.20$ for each $n=60,120,240$ and is not displayed in the plot.}
\label{fig:Simulations}
\vspace{-0.15in}
\end{figure}

\vspace{-0.1in}

\paragraph{Results}
The results for these simulations are displayed in Figure \ref{fig:Simulations}. We see that the local max norm results in lower error than any of the tested existing norms, across all the settings used.

\subsection{Movie ratings data}
\vspace{-0.1in}
We next compare several different matrix norms on two collaborative filtering movie ratings datasets, the Netflix \cite{Netflix} and MovieLens \cite{ML} datasets. The sizes of the data sets, and the split of the ratings into training, validation and test sets\footnote{
For Netflix, the test set we use is their ``qualification set'', designed for a more uniform distribution of ratings across users relative to the training set. For MovieLens, we choose our test set at random from the available data.
}, are:\vspace{-.2cm}
{\small \begin{center}
\begin{tabular}{c|c|c|c|c|c}
Dataset&\# users&\# movies&Training set&Validation set&Test set\\\hline
Netflix&480,189&17,770& 100,380,507 & 100,000 & 1,408,395 \\\hline
MovieLens&71,567&10,681& 8,900,054 & 100,000 & 1,000,000 \\
\end{tabular}
\end{center}}

\begin{table}[tdp]
\caption{\small Root mean squared error (RMSE) results for estimating movie ratings on Netflix and MovieLens data using a rank $30$ model. Setting $\tau=0$ corresponds to the uniform or  weighted or smoothed weighted trace norm (depending on $\zeta$), while $\tau=1$ corresponds to the max norm for any $\zeta$ value.}
\begin{center}
\begin{tabular}{|c|c|c|c|c|}
\hline
\multicolumn{5}{|c|}{MovieLens}\\
\hline
{$\zeta$} $\backslash$ {$\tau$}&0.00&0.05&0.10&1.00\\
\hline
0.00&0.7852&0.7827&0.7838&0.7918 \\
\hline
0.05&0.7836&{\bf 0.7822}&0.7842&---\\
\hline
0.10&0.7831&0.7837&0.7846&---\\
\hline
0.15&0.7833&0.7842&0.7854&---\\
\hline
0.20&0.7842&0.7853&0.7866&---\\
\hline
1.00&0.7997&\multicolumn{3}{c}{}\\
\cline{1-2}
\end{tabular}
\begin{tabular}{|c|c|c|c|c|}
\hline
\multicolumn{5}{|c|}{Netflix}\\
\hline
{$\zeta$} $\backslash$ {$\tau$}&   0.00&0.05&0.10&1.00 \\
\hline
0.00  & 0.9107 & 0.9092 & 0.9094 & 0.9131 \\
\hline
0.05         & 0.9095 & {\bf 0.9090} & 0.9107 &---\\
\hline
0.10               & 0.9096 & 0.9098& 0.9122 &---\\
\hline
0.15               & 0.9102& 0.9111& 0.9131&---\\
\hline
0.20               & 0.9126& 0.9344& 0.9153&---\\
\hline
1.00               & 0.9235&\multicolumn{3}{c}{}\\
\cline{1-2}
\end{tabular}
\end{center}
\label{tab:MovieData}
\vspace{-0.2in}
\end{table}%

We test the local max norm given in \eqref{eq:interpolate_exp} with $\zeta\in\{0,0.05,0.1,0.15,0.2\}$
and $\tau\in\{0,0.05,0.1\}$. We also test $\tau=1$ (the max norm---here $\zeta$ is arbitrary) and $\zeta=1,\tau=0$ (the uniform trace norm).
We follow the test protocol of \cite{FSSS}, with a rank-$30$ approximation to the optimization
problem \eqref{eq:alpha_tau_lambda}.

Table \ref{tab:MovieData} shows root mean squared error (RMSE) for the experiments
.
For both the MovieLens and Netflix data,
the local max norm with $\tau=0.05$ and $\zeta=0.05$
gives strictly better accuracy than any previously-known norm studied in this setting.
(In practice, we can use a validation set to reliably select good values for the $\tau$ and $\zeta$ parameters\footnote{
To check this, we subsample half of the test data at random, and use it as a validation set to choose $(\zeta,\tau)$ for each method (as specified in Table \ref{tab:MethodsForNorms}). We then evaluate error on the remaining half of the test data. For MovieLens, the local max norm gives an RMSE of 0.7820 with selected parameter values $\zeta=\tau=0.05$, as compared to an RMSE of 0.7829 with selected smoothing parameter $\zeta=0.10$ for the smoothed weighted trace norm. For Netflix, the local max norm gives an RMSE of 0.9093 with  $\zeta=\tau=0.05$, while the smoothed weighted trace norm gives an RMSE of 0.9098 with $\zeta=0.05$. The other tested methods give higher error on both datasets.
}.)
For the MovieLens data, the local max norm achieves RMSE
of $0.7822$, compared to $0.7831$ achieved by the smoothed empirically-weighted trace norm with $\zeta=0.10$,
which gives the best result among the previously-known norms. For the Netflix dataset
the local max norm achieves RMSE of 0.9090, improving upon the previous best result of 0.9096 achieved by the smoothed empirically-weighted trace norm \cite{FSSS}.

\section{Summary}
In this paper, we introduce a unifying family of matrix norms, called the ``local max'' norms, 
that generalizes existing methods for matrix reconstruction, such as the max norm and trace norm. We examine some interesting sub-families of local max norms, and consider several different options for interpolating between the trace (or smoothed weighted trace) and max norms. We find norms lying strictly between the trace norm and the max norm that give improved accuracy in matrix reconstruction for both simulated data and real movie ratings data. We show that regularizing with any local max norm is fairly simple to optimize, and give a theoretical result suggesting improved matrix reconstruction using new norms in this family.

\bibliographystyle{unsrt}
\bibliography{LMNbib}

\appendix

\section*{Supplementary Materials}
\section{Proof of Theorem 1}\label{SecA}

\paragraph{Special case: element-wise upper bounds}
First, we assume that the general result is true, i.e.\
\begin{equation}\label{eq:Factors_General}2\RCnorm{X}=\inf_{AB^{\top}=X}\left(\sup_{\r\in\Rset}\sum_i \r_i\ttnorm{A_{(i)}}+\sup_{\c\in\Cset}\sum_j \c_j \ttnorm{B_{(j)}}\right)\;,\end{equation}
and prove the result in the special case, where 
\[\Rset=\{\r\in\Del{n}: \r_i\leq  R_i \ \forall i\}\text{ and }\Cset=\{\c\in\Del{m}:\c_j\leq  C_j \ \forall j\}\;.\]
Using strong duality for linear programs, we have
\begin{align*}
\sup_{\r\in\Rset}\sum_i \r_i \ttnorm{A_{(i)}}
&=\sup_{\r\in\R^n_+}\left\{\sum_i \r_i \ttnorm{A_{(i)}} \ :\    \r_i\leq  R_i  , \ \sum_i \r_i=1\right\}\\
&=\inf_{a\in\R, a_1\in\R^n_+}\left\{ a+ R ^{\top}a_1 \ : \ a+a_{1i}\geq \ttnorm{A_{(i)}}\ \forall i\right\}\;.
\end{align*}
In this last line, if we fix $a$ and want to minimize over $a_1\in\R^n_+$, it is clear that  the infimum is obtained by setting
$a_{1i} =(\ttnorm{A_{(i)}}-a)_+$ for each $i$. This proves that
\[\sup_{\r\in\Rset}\sum_i \r_i \ttnorm{A_{(i)}}
=\inf_{a\in\R}\left\{ a+\sum_i  R_i \left(\ttnorm{A_{(i)}}-a\right)_+\right\}\;.\]
Applying the same reasoning to the columns and plugging everything in to \eqref{eq:Factors_General}, we get
\[2\RCnorm{X}=\inf_{{AB^{\top}=X, \ a,b\in\R}}\bigg\{a+\sum_i  R_i \left(\ttnorm{A_{(i)}}-a\right)_+  + b+\sum_j  C_j \left(\ttnorm{B_{(j)}}-b\right)_+\bigg\}\;.\]

\paragraph{General factorization result}
In the proof sketch given in the main paper, we showed that
\[
2\RCnorm{X}
\leq \inf_{AB^{\top}=X}\left(\sup_{\r\in\Rset}\frnorm{\dh{\r}A}^2+\sup_{\c\in\Cset}\frnorm{\dh{\c}B}^2\right)\;.\]
We now want to prove the reverse inequality. Since $\RCnorm{X}=\norm{X}_{(\overline{\Rset},\overline{\Cset})}$ by definition (where $\overline{\mathcal{S}}$ denotes the closure of a set $\mathcal{S}$), we can assume without loss of generality that $\Rset$ and $\Cset$ are both closed (and compact) sets.

First, we restrict our attention to a special case (the ``positive case''), where we assume that for all $\r\in\Rset$ and all $\c\in\Cset$, $\r_i>0$ and $\c_j>0$ for all $i$ and $j$. (We will treat the general case below.)
Therefore, since $\rcnorm{X}$ is continuous as a function of $(\r,\c)$ for any fixed $X$ and since $\Rset$ and $\Cset$ are closed, we must have some $\r^{\star}\in\Rset$ and $\c^{\star}\in\Cset$ such that $\RCnorm{X}=\wtrnorm{\r^{\star},\c^{\star}}{X}$, with $\r^{\star}_i>0$ for all $i$ and $\c^{\star}_j>0$ for all $j$.

Next, let $UDV^{\top}=\dh{\r^{\star}}\cdot X\cdot \dh{\c^{\star}}$ be a singular value decomposition, and let
 $A^{\star}=\dhi{\r^{\star}}U\dh{D}$ and $B^{\star}=\dhi{\c^{\star}}V\dh{D}$. Then $A^{\star}B^{\star}{}^{\top}=X$, and 
\[\frnorm{\dh{\r^{\star}}A^{\star}}^2=\frnorm{U\dh{D}}^2=\trace(UDU^{\top})=\trace(D)=\wtrnorm{\r^{\star},\c^{\star}}{X}=\RCnorm{X}\;.\]
Below, we will show that 
\begin{equation}\label{eq:Astep}\r^{\star}=\arg\max_{\r\in\Rset}\frnorm{\dh{\r}A^{\star}}^2\;.\end{equation}
This will imply that $\RCnorm{X}=\sup_{\r\in\Rset}\frnorm{\dh{\r}A^{\star}}^2$, and following the same reasoning for $B^{\star}$, we will have proved
\[2\RCnorm{X}=\left(\sup_{\r\in\Rset}\frnorm{\dh{\r}A^{\star}}^2+\sup_{\c\in\Cset}\frnorm{\dh{\c}B^{\star}}^2\right)\geq \inf_{AB^{\top}=X}\left(\sup_{\r\in\Rset}\frnorm{\dh{\r}A}^2+\sup_{\c\in\Cset}\frnorm{\dh{\c}B}^2\right)\;,\]
which is sufficient. It remains only to prove \eqref{eq:Astep}. Take any $\r\in\Rset$ with $\r\neq \r^{\star}$ and let $\w=\r-\r^{\star}$. We have
\[
\frnorm{\dh{\r}A}^2-\frnorm{\dh{\r^{\star}}A}^2
=\sum_i \w_i \ttnorm{A_{(i)}} \\
=\sum_i \frac{\w_i}{\r^{\star}_i} \cdot(UDU^{\top})_{ii}\;,\]
and it will be sufficient to prove that this quantity is $\leq 0$. To do this, we first define, for any $t\in[0,1]$,
\[f(t)\coloneqq \sum_i \sqrt{1+t\cdot \frac{\w_i}{\r^{\star}_i}} \cdot (UDU^{\top})_{ii}= \trace\left(\dh{\left(\frac{\r^{\star}+t\w}{\r^{\star}}\right)} UDU^{\top}\right)\;.\]
Using the fact that $\trace(\cdot)\leq \trnorm{\cdot}$ for all matrices, we have
\begin{multline*}
f(t)  \leq 
\trnorm{\dh{\left(\frac{\r^{\star}+t\w}{\r^{\star}}\right)} UDU^{\top}}
=\trnorm{\dh{(\r^{\star}+t\w)}X\dh{\c^{\star}}\cdot VU^{\top}}  \\
= \trnorm{\dh{(\r^{\star}+t\w)}X\dh{\c^{\star}}}=\wtrnorm{\r^{\star}+t\w,\c^{\star}}{X}\leq \RCnorm{X}=\sum_i (UDU^{\top})_{ii}=f(0)\;,
\end{multline*}
where the last inequality comes from the fact that $\r^{\star}+t\w\in\Rset$ by convexity of $\Rset$. Therefore,
\[0\geq  \frac{d}{dt}\;f(t)\bigg\vert_{t=0} = \frac{d}{dt}\left(\sum_i \sqrt{1+t\cdot \frac{\w_i}{\r^{\star}_i}} \cdot (UDU^{\top})_{ii}\right) \ \bigg\vert_{t=0} = \frac{1}{2}\cdot\sum_i \frac{\w_i}{\r^{\star}_i} \cdot (UDU^{\top})_{ii} \;,\]
as desired. (Here we take the right-sided derivative, i.e.\ taking a limit as $t$ approaches zero from the right, since $f(t)$ is only defined for $t\in[0,1]$.) This concludes the proof for the positive case.

Next, we prove that the general factorization \eqref{eq:Factors_General} hold in the general case, where we might have $\overline{\Rset}\not\subset\R^n_{++}$ and/or $\overline{\Cset}\not\subset \R^m_{++}$. If for any $i\in[n]$ we have $\r_i=0$ for all $\r\in\Rset$, we can discard this row of $X$, and same for any $j\in[m]$. Therefore, without loss of generality, for all $i\in[n]$ there is some $\r^{(i)}\in\Rset$ with $\r^{(i)}_i>0$. Taking a convex combination, $\r^+= \frac{1}{n}\sum_i \r^{(i)}\in\Rset$, we have $\r^+\in\Rset\cap \R^n_{++}$. Similarly, we can construct $\c^+\in\Cset\cap\R^m_{++}$.

Fix any $\eps>0$, and let $\del=\min\{\min_i \r^+_i,\min_j \c^+_j\}\cdot \frac{\eps}{2(1+\eps)}>0$, and define closed subsets
\[\Rset_0=\left\{\r\in\Rset:\min_i \r_i\geq \del\right\}\subseteq \Rset\text{ and }\Cset_0=\left\{\c\in\Cset:\min_i \c_i\geq \del\right\}\subseteq \Cset\;.\]
Since we know that the factorization result holds for the ``positive case'', we have
\begin{multline*} \inf_{AB^{\top}=X}\left(\sup_{\r\in\Rset_0}\frnorm{\dh{\r}A}^2+\sup_{\c\in\Cset_0}\frnorm{\dh{\c}B}^2\right)=2\norm{X}_{(\Rset_0,\Cset_0)}\\=2\sup_{\r\in\Rset_0,\c\in\Cset_0}\trnorm{\dh{\r}X\dh{\c}}\leq 2\sup_{\r\in\Rset,\c\in\Cset}\trnorm{\dh{\r}X\dh{\c}}=2\RCnorm{X}\;.\end{multline*}
Now choose any factorization $\tilde{A}\tilde{B}^{\top}=X$ such that
\begin{equation}\label{tildeAB}\left(\sup_{\r\in\Rset_0}\frnorm{\dh{\r}\tilde{A}}^2+\sup_{\c\in\Cset_0}\frnorm{\dh{\c}\tilde{B}}^2\right)\leq2 \sup_{\r\in\Rset,\c\in\Cset}\trnorm{\dh{\r}X\dh{\c}}(1+\nicefrac{\eps}{2})\;.\end{equation}

Next, we need to show that $\sup_{\r\in\Rset}\frnorm{\dh{\r}\tilde{A}}^2$ is not much larger than $\sup_{\r\in\Rset_0}\frnorm{\dh{\r}\tilde{A}}^2$ (and same for $\tilde{B}$). Choose any $\r'\in\Rset$, and let $\r''=\left(1-\frac{\del}{\min_i \r^+_i}\right)\r'+\left(\frac{\del}{\min_i \r^+_i}\right) \r^+\in\Rset$. Then
\[\min_i \r''_i\geq\left(\frac{\del}{\min_i \r^+_i}\right) \min_i \r^+_i=\del\;,\]
and so $\r''\in\Rset_0$. We also have $\r'_i\leq \left(1-\frac{\del}{\min_i \r^+_i}\right)^{-1}\r''_i$ for all $i$.
Therefore,
\[\frnorm{\dh{\r'}\tilde{A}}\leq \left(1-\frac{\del}{\min_i \r^+_i}\right)^{-\nicefrac{1}{2}}\frnorm{\dh{\r''}\tilde{A}} \leq\left(1-\frac{\del}{\min_i \r^+_i}\right)^{-\nicefrac{1}{2}}\sup_{\r\in\Rset_0}\frnorm{\dh{\r}\tilde{A}}\;.\]
Since this is true for any $\r'\in\Rset$, applying the definition of $\del$, we have
\[\sup_{\r\in\Rset}\frnorm{\dh{\r}\tilde{A}}\leq \left(1-\frac{\del}{\min_i \r^+_i}\right)^{-\nicefrac{1}{2}}\sup_{\r\in\Rset_0}\frnorm{\dh{\r}\tilde{A}}\leq \left(\frac{1+\nicefrac{\eps}{2}}{1+\eps}\right)^{-\nicefrac{1}{2}}\sup_{\r\in\Rset_0}\frnorm{\dh{\r}\tilde{A}}\;.\]

Applying the same reasoning for $\tilde{B}$ and then plugging in the bound \eqref{tildeAB}, we have
\begin{multline*}
\inf_{AB^{\top}=X}\left(\sup_{\r\in\Rset}\frnorm{\dh{\r}{A}}^2+\sup_{\c\in\Cset}\frnorm{\dh{\c}{B}}^2\right)
\leq \left(\sup_{\r\in\Rset}\frnorm{\dh{\r}\tilde{A}}+\sup_{\c\in\Cset}\frnorm{\dh{\c}\tilde{B}}^2\right)\\
\leq \left(\frac{1+\nicefrac{\eps}{2}}{1+\eps}\right)^{-1}\cdot \left(\sup_{\r\in\Rset_0}\frnorm{\dh{\r}\tilde{A}}^2+\sup_{\c\in\Cset_0}\frnorm{\dh{\c}\tilde{B}}^2\right)\\
\leq  \left(\frac{1+\nicefrac{\eps}{2}}{1+\eps}\right)^{-1}(1+\nicefrac{\eps}{2})\cdot2 \RCnorm{X}
=  (1+\eps)\cdot2 \RCnorm{X}\;.\end{multline*}
Since this analysis holds for arbitrary $\eps>0$, this proves the desired result, that
\[\inf_{AB^{\top}=X}\left(\sup_{\r\in\Rset}\frnorm{\dh{\r}{A}}^2+\sup_{\c\in\Cset}\frnorm{\dh{\c}{B}}^2\right) \leq 2\RCnorm{X}\;.\]

\section{Proof of Theorem 2}

We follow similar techniques as used by Srebro and Shraibman \cite{ShraibmanSrebro} in their proof of the analogous result for the max norm. We  need to show that
 \begin{multline*}\conv{uv^{\top}: u\in\R^n,v\in\R^m,\Rnorm{u}=\Cnorm{v}=1}\subseteq \left\{X:\RCnorm{X}\leq1 \right\}\subseteq \\K_G\cdot \conv{uv^{\top}: u\in\R^n,v\in\R^m,\Rnorm{u}=\Cnorm{v}=1}\;.\end{multline*}
For the left-hand inclusion, since $\RCnorm{\cdot}$ is a norm and therefore the constraint $\RCnorm{X}\leq1$ is convex, it is sufficient to show that $\RCnorm{uv^{\top}}\leq 1$ for any $u\in\R^n,v\in\R^m$ with $\Rnorm{u}=\Cnorm{v}=1$. This is a trivial consequence of the factorization result in Theorem 1.

Now we prove the right-hand inclusion.
Grothendieck's Inequality states that, for any $Y\in\R^{n\times m}$ and for any dimension $k$,
\begin{multline*}\sup\left\{\langle Y, UV^{\top}\rangle :U\in\R^{n\times k},V\in\R^{m\times k}, \tnorm{U_{(i)}}\leq 1 \ \forall i,\  \tnorm{V_{(j)}}\leq 1\ \forall j\right\}\\\leq K_G\cdot \sup\left\{\langle Y, uv^{\top}\rangle : u\in\R^n, v\in\R^m, |u_i|\leq 1 \ \forall i,\  |v_j|\leq 1\ \forall j\right\}\;,\end{multline*}
where $K_G\in(1.67,1.79)$ is Grothendieck's constant. We now extend this to a slightly more general form. Take any $a\in\R^n_+$ and $b\in\R^m_+$. Then, setting $\tilde{U}=\diag(a)^+U$ and $\tilde{V}=\diag(b)^+V$ (where $M^+$ is the pseudoinverse of $M$), and same for $\tilde{u}$ and $\tilde{v}$, we see that
\begin{multline}\label{eq:WeightedGrothendieck}\sup\left\{\langle Y, UV^{\top}\rangle :U\in\R^{n\times k},V\in\R^{m\times k}, \tnorm{U_{(i)}}\leq a_i \ \forall i,\  \tnorm{V_{(j)}}\leq b_j\ \forall j\right\}\\
= \sup\left\{\langle \diag(a)\cdot Y\cdot\diag(b), \tilde{U}\tilde{V}^{\top}\rangle : \tilde{U}\in\R^{n\times k},\tilde{V}\in\R^{m\times k}, \tnorm{\tilde{U}_{(i)}}\leq 1 \ \forall i,\  \tnorm{\tilde{V}_{(j)}}\leq 1\ \forall j\right\}\\
\leq K_G\cdot  \sup\left\{\langle \diag(a)\cdot Y\cdot\diag(b), \tilde{u}\tilde{v}^{\top}\rangle : \tilde{u}\in\R^n,\tilde{v}\in\R^m, |\tilde{u}_i|\leq 1 \ \forall i,\  |\tilde{v}_j|\leq 1\ \forall j\right\}\\
=K_G\cdot  \sup\left\{\langle Y,uv^{\top}\rangle : u\in\R^n,v\in\R^m, |u_i|\leq a_i \ \forall i,\  |v_j|\leq b_j\ \forall j\right\}\;.
\end{multline}

Now take any $Y\in\R^{n\times m}$. Let $\RCnorm{\cdot}^*$ be the dual norm to the $(\Rset,\Cset)$-norm. To bound this dual norm of $Y$, we apply the factorization result of Theorem 1:
\begin{align*}
\RCnorm{Y}^*
&=\sup_{\RCnorm{X}\leq 1}\langle Y,X\rangle\\
&=\sup_{U,V}\left\{\langle Y, UV^{\top}\rangle : \frac{1}{2}\left(\sup_{\r\in\Rset}\sum_i \r_i \ttnorm{U_{(i)}} + \sup_{\c\in\Cset}\sum_j \c_j \ttnorm{V_{(j)}}\right)\leq 1\right\}\\
&\stackrel{(*)}{=}\sup_{U,V}\left\{\langle Y, UV^{\top}\rangle : \sup_{\r\in\Rset}\sum_i \r_i \ttnorm{U_{(i)}}= \sup_{\c\in\Cset}\sum_j \c_j \ttnorm{V_{(j)}}\leq 1\right\}\\
&=\sup_{\substack{a\in\R^n_+:\Rnorm{a}\leq 1\\b\in\R^m_+:\Cnorm{b}\leq 1}}\sup_{U,V}\left\{\langle Y, UV^{\top}\rangle : \tnorm{U_{(i)}}\leq  a_i \ \forall i, \ \tnorm{V_{(j)}} \leq b_j \ \forall j\right\}\\
&\leq K_G\cdot \sup_{\substack{a\in\R^n_+:\Rnorm{a}\leq 1\\b\in\R^m_+:\Cnorm{b}\leq 1}}\sup_{U,V}\left\{\langle Y, uv^{\top}\rangle : |u_i|\leq  a_i \ \forall i, \ |v_j|\leq b_j \ \forall j\right\}\\
&= K_G\cdot \sup_{u,v}\left\{\langle Y, uv^{\top}\rangle : \Rnorm{u}\leq 1,\Cnorm{v}\leq 1\right\}\\
&= K_G\cdot \sup_X\left\{\langle Y, X\rangle : X\in \conv{uv^{\top}: u\in\R^n,v\in\R^m,\Rnorm{u}=\Cnorm{v}=1}\right\}\\
&= \sup_X\left\{\langle Y, X\rangle : X\in K_G\cdot \conv{uv^{\top}: u\in\R^n,v\in\R^m,\Rnorm{u}=\Cnorm{v}=1}\right\}\;.
\end{align*}
As in \cite{ShraibmanSrebro}, this is sufficient to prove the result.
Above, the step marked (*) is true because, given any $U$ and $V$ with
\[\frac{1}{2}\left(\sup_{\r\in\Rset}\sum_i \r_i \ttnorm{U_{(i)}} + \sup_{\c\in\Cset}\sum_j \c_j \ttnorm{V_{(j)}}\right)\leq 1\;,\]
we can replace $U$ and $V$ with $U'\coloneqq U\cdot \omega$ and $V'\coloneqq V\cdot \omega^{-1}$, where
$\omega \coloneqq \sqrt[4]{\frac{\sup_{\c\in\Cset}\sum_j \c_j \ttnorm{V_{(j)}}}{\sup_{\r\in\Rset}\sum_i \r_i \ttnorm{U_{(i)}}}}$.
This will give $U'V'^{\top}=UV^{\top}$, and 
\begin{multline*}\sup_{\r\in\Rset}\sum_i \r_i \ttnorm{U'_{(i)}} = \sup_{\c\in\Cset}\sum_j \c_j \ttnorm{V'_{(j)}} = \sqrt{\sup_{\r\in\Rset}\sum_i \r_i \ttnorm{U_{(i)}}\cdot  \sup_{\c\in\Cset}\sum_j \c_j \ttnorm{V_{(j)}} }  \\\leq \frac{1}{2}\left(\sup_{\r\in\Rset}\sum_i \r_i \ttnorm{U_{(i)}} + \sup_{\c\in\Cset}\sum_j \c_j \ttnorm{V_{(j)}}\right)\leq 1\;.\end{multline*}

\section{Proof of Theorem 3}

Following the strategy of Srebro \& Shraibman (2005), we will use the Rademacher complexity to bound this excess risk. By Theorem 8 of Bartlett \& Mendelson (2002)\footnote{
The statement of their theorem gives a result that holds with high probability, but in the proof of this result they derive a bound in expectation, which we use here.}, we know that
\begin{multline}\label{eq:RadToBound}\Ep{S}{\sum_{ij}\pij \left|Y_{ij}-\widehat{X}_{ij}\right|- \inf_{\RCnorm{X}\leq \sqrt{k}}\sum_{ij}\pij \left|Y_{ij}-X_{ij}\right|}\\=   \O{ \Ep{S}{\empRad{S}{\left\{X\in\R^{n\times m}:\RCnorm{X}\leq \sqrt{k}\right\}}}}\;,\end{multline}
where the expected Rademacher complexity is defined as
\[\Ep{S}{\empRad{S}{\left\{X\in\R^{n\times m}:\RCnorm{X}\leq \sqrt{k}\right\}}}\coloneqq \frac{1}{s}\Ep{S,\nu}{\sup_{\RCnorm{X}\leq \sqrt{k}}\sum_t \nu_t \cdot X_{i_tj_t}}\;,\]
where $\nu\in\{\pm 1\}^s$ is a random vector of independent unbiased signs, generated independently from $S$.

Now we bound the Rademacher complexity. By scaling, it is sufficient to consider the case $k=1$. The main idea for this proof is to first show that, for any $X$ with $\RCnorm{X}\leq1$, we can decompose $X$  into a sum $X'+X''$ where $\maxnorm{X'}\leq K_G$ and $\wtrnorm{\tp}{X''}\leq 2K_G\gam^{-\nicefrac{1}{2}}$, where $\tp$ represents the smoothed row and column marginals with smoothing parameter $\zeta=\nicefrac{1}{2}$, and where $K_G\leq 1.79$ is Grothendieck's constant. We will then use known Rademacher complexity bounds for the classes of matrices that have bounded max norm and bounded smoothed weighted trace norm.  

To construct the decomposition of $X$, we start with a vector decomposition lemma, proved below.
\begin{lemma}\label{lem:Decompose_vector}
Suppose $\Rset\supseteq \Rset^{\times}_{\nicefrac{1}{2},\gam}$. Then for any $u\in\R^n$ with $\Rnorm{u}=1$, we can decompose $u$ into a sum $u=u'+u''$ such that 
$\norm{u'}_{\infty}\leq 1$ and $\norm{u''}_{\tprow}\coloneqq\sum_i \tprowi u_i''{}^2\leq \gam^{-\nicefrac{1}{2}}$.\end{lemma}
Next, by Theorem 2, we can write
\[X=K_G\cdot \sum_{l=1}^{\infty} t_l \cdot u_lv_l^{\top}\;,\]
where $t_l\geq 0$, $\sum_{l=1}^{\infty}t_l=1$, and $\Rnorm{u_l}=\Cnorm{v_l}=1$ for all $l$. Applying Lemma \ref{lem:Decompose_vector} to $u_l$ and to $v_l$ for each $l$, we can write $u_l=u_l'+u_l''$ and $v_l=v_l'+v_l''$, where 
\[\norm{u_l'}_{\infty}\leq 1, \ \norm{u_l''}_{\tprow}\leq\gam^{-\nicefrac{1}{2}},\  \norm{v_l'}_{\infty}\leq1, \ \norm{v_l''}_{\tpcol}\leq\gam^{-\nicefrac{1}{2}}\;.\]
Then
\[X=K_G\cdot \left(\sum_{l=1}^{\infty} t_l\cdot u_l'v_l'{}^{\top} + \sum_{l=1}^{\infty} t_l\cdot u_l'v_l''{}^{\top} + \sum_{l=1}^{\infty} t_l\cdot u_l''v_l{}^{\top} \right)=:K_G\left(X_1+X_2+X_3\right)\;.\]
Furthermore, $\norm{u_l'}_{\tprow}\leq\norm{u_l'}_{\infty}\leq 1$, and $\norm{v_l}_{\tprow}\leq \norm{v_l}_{\Cset}\leq 1$. 
Applying Srebro and Shraibman \cite{ShraibmanSrebro}'s convex hull bounds for the trace norm and max norm (stated in Section 4 of the main paper), we see that $\maxnorm{X_1}\leq 1$, and that that $\wtrnorm{\tp}{X_i}\leq \gam^{-\nicefrac{1}{2}}$ for $i=2,3$.
Defining $X'=X_1$ and $X''=X_2+X_3$, we have the desired decomposition.

Applying this result to every $X$ in the class $\left\{X\in\R^{n\times m}:\RCnorm{X}\leq 1\right\}$, we see that
\begin{multline*}\Ep{S}{\empRad{S}{\left\{X\in\R^{n\times m}:\RCnorm{X}\leq 1\right\}}}\\\leq
\Ep{S}{\empRad{S}{\left\{X':\maxnorm{X'}\leq K_G\right\}}}+\Ep{S}{\empRad{S}{\left\{X'':\wtrnorm{\tp}{X''}\leq K_G\cdot 2\gam^{-\nicefrac{1}{2}}\right\}}}\\
\leq K_G\cdot   \O{\sqrt{\frac{n}{s}}} +K_G\cdot 2\gam^{-\nicefrac{1}{2}}\cdot  \O{\sqrt{\frac{n\log(n)}{s}}+\frac{n\log(n)}{s}} \;,\end{multline*}
where the last step uses bounds on the Rademacher complexity of the max norm and weighted trace norm unit balls, shown in Theorem 5 of \cite{ShraibmanSrebro} and Theorem 3 of \cite{FSSS}, respectively. Finally, we want to deal with the last term, $\frac{n\log(n)}{s}$, that is outside the square root. Since $s\geq n$ by assumption, we have
$\frac{n\log(n)}{s} \leq \sqrt{\frac{n\log^2(n)}{s}}$,
and if $s\geq n\log(n)$, then we can improve this to $\frac{n\log(n)}{s} \leq \sqrt{\frac{n\log(n)}{s}}$.
Returning to \eqref{eq:RadToBound} and plugging in our bound on the Rademacher complexity, this proves the desired bound on the excess risk.

\subsection{Proof of  Lemma \ref{lem:Decompose_vector}}

For $u\in\R^n$ with $\Rnorm{u}=1$, we need to find a decomposition $u=u'+u''$ such that 
$\norm{u'}_{\infty}\leq 1$ and $\norm{u''}_{\tprow}=\sqrt{\sum_i \tprowi u_i''{}^2}\leq \gam^{-\nicefrac{1}{2}}$.
Without loss of generality, assume 
$|u_1|\geq \dots\geq |u_n|$. Find $N\in\{1,\dots,n\}$ and $t\in(0,1]$ so that
$\sum_{i=1}^{N-1} \tprowi + t\cdot \tp_{N\bul}= \gam^{-1}$, and let 
\[\r=\gam\cdot (\tp_{1\bul},\dots,\tp_{(N-1)\bul},t\cdot\tp_{N\bul},0,\dots,0)\in\Del{n}\;.\]
 Clearly, $\r_i\leq \gam\cdot \tprowi$ for all $i$, and so $\r\in\Rset^{\times}_{\nicefrac{1}{2},\gam}\subseteq\Rset$. 

Now let
$u''=(u_1,\dots,u_{N-1},\sqrt{t}\cdot u_N,0,\dots,0)$, and set $u'=u-u''$.
 We then calculate
\[\norm{u''}_{\tprow}^2= \sum_{i=1}^{N-1} \tprowi u_i^2 + t\cdot \tp_{N\bul} u_N^2 = \gam^{-1} \sum_{i=1}^n \r_i u_i^2\leq\gam^{-1} \Rnorm{u}^2\leq \gam^{-1}\;.\]
Finally, we want to show that $\norm{u'}_{\infty}\leq 1$. Since $u'_i=0$ for $i<N$, we only need to bound $|u'_i|$ for each $i\geq N$. We have
\[1= \Rnorm{u}^2 \geq \sum_{i'=1}^n \r_{i'} u_{i'}^2\geq \sum_{i'=1}^N \r_{i'} u_{i'}^2 \stackrel{(*)}{\geq} u_i^2\cdot \sum_{i'=1}^N\r_{i'} \stackrel{(\#)}{=} u_i^2\geq u'_i{}^2\;,\]
where the step marked (*) uses the fact that $|u_{i'}|\geq|u_i|$ for all $i'\leq N$, and the step marked (\#) comes from the fact that $\r$ is supported on $\{1,\dots,N\}$. This is sufficient.

\section{Proof of Proposition 1}
Let $L_0=\mathrm{Loss}(\widehat{X})$. Then, by definition, 
\[\widehat{X}=\arg\min\left\{\mathrm{Penalty}_{(\beta,\tau)}(X) : \mathrm{Loss}(X)\leq L_0\right\}\;.\]
Then to prove the lemma, it is sufficient to show that for some $t\in[0,1]$,
\[\widehat{X}=\arg\min\left\{\norm{X}_{(\Rset_{(t)},\Cset_{(t)})} : \mathrm{Loss}(X)\leq L_0\right\}\;,\]
where we set
\[\Rset_{(t)}=\left\{\r\in\Del{n}:\r_i\geq \frac{t}{1+(n-1) t} \ \forall i\right\}, \ \Cset_{(t)}=\left\{\c\in\Del{m}:\c_j\geq\frac{t}{1+(m-1)t}\ \forall j\right\}\;.\]
Trivially, we can rephrase these definitions as 
\begin{multline}\label{eq:Rset_t}\Rset_{(t)}=\left\{\frac{t}{1+(n-1)\cdot t}\cdot (1,\dots,1) + \frac{1- t}{1+(n-1)\cdot t}\cdot \r:\r\in\Del{n}\right\}\text{ and } \\\Cset_{(t)}=\left\{\frac{ t}{1+(m-1)\cdot t}\cdot(1,\dots,1)+\frac{1-t}{1+(m-1)\cdot t}\cdot\c:\c\in\Del{m}\right\}\;.\end{multline}

Note that for any vectors $u\in\R^n_+$ and $v\in\R^m_+$,
\begin{equation}\label{MaxIdent}\sup_{\r\in\Del{n}}\sum_i \r_i u_i=\max_i u_i \text{ and }\sup_{\c\in\Del{m}}\sum_j \c_j v_j=\max_j v_j\;.\end{equation}
Applying the SDP formulation of the local max norm (proved in Lemma \ref{thm:SDP} below), we have
\begin{multline}\label{SDP_RCt}\norm{X}_{(\Rset_{(t)},\Cset_{(t)})}=\frac{1}{2}\inf\left\{\sup_{\r\in\Rset_{(t)}}\sum_i \r_i  U _{ii}+\sup_{\c\in\Cset_{(t)}}\sum_j \c_j  V _{jj}:\left(\begin{array}{cc} U &X\\X^{\top}& V \end{array}\right)\succeq 0\right\}\\
\stackrel{\text{By \eqref{eq:Rset_t} and \eqref{MaxIdent}}}{=}\frac{1}{2}\inf\bigg\{\frac{ t}{1+(n-1)\cdot t}\cdot \sum_i  U _{ii} +  \frac{1- t}{1+(n-1)\cdot t}\max_i U _{ii}\\+\frac{ t}{1+(m-1)\cdot t}\cdot \sum_j  V _{jj} 
+  \frac{1- t}{1+(m-1)\cdot t}\max_j V _{jj}:\left(\begin{array}{cc} U &X\\X^{\top}& V \end{array}\right)\succeq 0\bigg\}\\
=\frac{\omega_t }{2}\inf\bigg\{t \sum_i  A _{ii} + (1-t)\max_i   A _{ii}+t \sum_j  B _{jj} 
+ (1-t)\max_j  B _{jj}:\left(\begin{array}{cc} A &X\\X^{\top}& B \end{array}\right)\succeq 0\bigg\}\\
=\frac{\omega_t }{2}\inf\bigg\{(1-t)\cdot \MM(A,B) + t\cdot \TT(A,B):X\in\xab\bigg\}\;,
\end{multline}
where for the next-to-last step, we define
\[A=U\cdot\sqrt{\frac{1+(m-1)\cdot t}{1+(n-1)\cdot t}}, \ B = V\cdot\sqrt{\frac{1+(n-1)\cdot t}{1+(m-1)\cdot t}}, \ \omega_t =\frac{1}{\sqrt{(1+(n-1)\cdot t)(1+(m-1)\cdot t)}}\;,\]
and for the last step, we define
\[\TT(A,B)=\trace(A)+\trace(B), \ \MM(A,B)=\max_i A_{ii}+\max_j B_{jj}\;,\]
and
\[\xab=\left\{X:\left(\begin{array}{cc} A &X\\X^{\top}& B \end{array}\right)\succeq 0\right\}\;.\]

Next, we compare this to the $(\beta,\tau)$ penalty formulated in our main paper. Recall 
\[\mathrm{Penalty}_{(\beta,\tau)}(X)= \inf_{X=AB^{\top}}\left\{\sqrt{\max_i \norm{A_{(i)}}^2_2+\max_j \norm{B_{(j)}}^2_2}\cdot \sqrt{\sum_i \norm{A_{(i)}}^2_2+\sum_j \norm{B_{(j)}}^2_2}\right\}\;.\]
Applying Lemma \ref{lem:Factor_vs_SDP} below, we  can obtain an equivalent SDP formulation of the penalty
\begin{equation}\label{decomp_AB}\mathrm{Penalty}_{(\beta,\tau)}(X)
=\inf_{A,B}\left\{\sqrt{\MM(A,B)}\cdot\sqrt{\TT(A,B)} \ : \ X\in\xab\right\}\;.\end{equation}
Since $\MM(A,B)\leq\TT(A,B)\leq\max\{n,m\}\MM(A,B)$, and since for any $x,y> 0$ we know $\sqrt{xy}\leq\frac{1}{2}\left(\alpha \cdot x + \alpha^{-1}\cdot y\right)$ for any $\alpha>0$ with equality attained when $\alpha=\sqrt{y/x}$, we see that
\begin{multline*}\mathrm{Penalty}_{(\beta,\tau)}(\widehat{X})
=\frac{1}{2}\inf_{A,B}\left\{\inf_{\alpha\in[1,\sqrt{\max\{n,m\}}]}\left\{\alpha\cdot\MM(A,B)+\alpha^{-1}\cdot\TT(A,B)\right\} \ : \ \widehat{X}\in\xab\right\}\\
=\inf_{\alpha\in[1,\sqrt{\max\{n,m\}}]}\left[\frac{1}{2}\inf_{A,B}\left\{\alpha\cdot\MM(A,B)+\alpha^{-1}\cdot\TT(A,B) \ : \ \widehat{X}\in\xab\right\}\right]
\;.\end{multline*}
Since the quantity inside the square brackets is nonnegative and is continuous in $\alpha$, and we are minimizing over $\alpha$ in a compact set, the infimum is attained at some $\widehat{\alpha}$, so we can write
\[\mathrm{Penalty}_{(\beta,\tau)}(\widehat{X})
=\frac{1}{2}\inf_{A,B}\left\{\widehat{\alpha}\cdot\MM(A,B)+\widehat{\alpha}^{-1}\cdot\TT(A,B) \ : \ \widehat{X}\in\xab\right\}\;.\]
Recall that $\widehat{X}$ minimizes $\mathrm{Penalty}_{(\beta,\tau)}(X)$ subject to the constraint $\mathrm{Loss}(X)\leq L_0$. Setting $t\coloneqq\frac{\widehat{\alpha}^{-1}}{\widehat{\alpha}+\widehat{\alpha}^{-1}
}$, we get
\begin{multline*}\widehat{X}\in
\arg\min_X\left\{\inf_{A,B}\left\{\widehat{\alpha}\cdot\MM(A,B)+\widehat{\alpha}^{-1}\cdot\TT(A,B) \ : \ X\in\xab\right\} : \mathrm{Loss}(X)\leq L_0\right\}\\
=\arg\min_X\left\{\inf_{A,B}\left\{\frac{\widehat{\alpha}}{\widehat{\alpha}+\widehat{\alpha}^{-1}}\cdot\MM(A,B)+\frac{\widehat{\alpha}^{-1}}{\widehat{\alpha}+\widehat{\alpha}^{-1}
}\cdot\TT(A,B) \ : \ X\in\xab\right\} : \mathrm{Loss}(X)\leq L_0\right\}\\
 = \arg\min_X\left\{\inf_{A,B}\left\{(1-t)\cdot\MM(A,B)+t\cdot\TT(A,B) \ : \ X\in\xab\right\} : \mathrm{Loss}(X)\leq L_0\right\}\\
 = \arg\min_X\left\{\norm{X}_{(\Rset_{(t)},\Cset_{(t)})}: \mathrm{Loss}(X)\leq L_0\right\}\;,\end{multline*}
as desired.

\section{Computing the local max norm with an SDP}

\begin{lemma}\label{thm:SDP}
Suppose $\Rset$ and $\Cset$ are convex, and are defined by SDP-representable  constraints. Then the $(\Rset,\Cset)$-norm can be calculated with the semidefinite program
\[\RCnorm{X}=\frac{1}{2}\inf\left\{\sup_{\r\in\Rset}\sum_i \r_i  A _{ii}+\sup_{\c\in\Cset}\sum_j \c_j  B _{jj}:\left(\begin{array}{cc} A &X\\X^{\top}& B \end{array}\right)\succeq 0\right\}\;.\] 
In the special case where $\Rset$ and $\Cset$ are defined as in (8)  
in the main paper, then the norm is given by
\begin{multline*}
\RCnorm{X}=\frac{1}{2}\inf\bigg\{a+ R ^{\top}a_1 + b+  C ^{\top}b_1 \ : \
 a_{1i}\geq 0  \text{ and }  a+a_{1i}\geq  A _{ii} \ \forall i,\\ b_{1j}\geq 0 \text{ and } b+b_{1j}\geq  B _{jj} \ \forall j,\left(\begin{array}{cc} A &X\\X^{\top}& B \end{array}\right)\succeq 0\bigg\}\;.
\end{multline*}
\end{lemma}

\begin{proof}
For the general case, based on Theorem 1 in the main paper, we only need to show that
\begin{multline*}\inf\left\{\sup_{\r\in\Rset}\sum_i \r_i  A _{ii}+\sup_{\c\in\Cset}\sum_j \c_j  B _{jj}:\left(\begin{array}{cc} A &X\\X^{\top}& B \end{array}\right)\succeq 0\right\}\\
=\inf\left(\sup_{\r\in\Rset}\sum_i \r_i\ttnorm{A_{(i)}}+\sup_{\c\in\Cset}\sum_j \c_j \ttnorm{B_{(j)}}:AB^{\top}=X\right)\;.\end{multline*}
This is proved in Lemma \ref{lem:Factor_vs_SDP} below. 

For the special case where $\Rset$ and $\Cset$ are defined by element-wise bounds, we return to the proof  of Theorem 1 given in Section \ref{SecA}, where we see that
\[2\RCnorm{X}=\inf_{\substack{AB^{\top}=X, a,b\in\R \\ a_1\in\R^n_+, b_1\in\R^m_+}}\bigg\{a+ R ^{\top}a_1  + b+ C ^{\top}b_1 \ : \ a+a_{1i}\geq \ttnorm{A_{(i)}}\ \forall i, \ b+b_{1j}\geq \ttnorm{B_{(j)}}\ \forall j\bigg\}\;.\]
 Noting that $\ttnorm{A_{(i)}}=(AA^{\top})_{ii}$ and $ \ttnorm{B_{(j)}}=(BB^{\top})_{jj}$, we again use Lemma \ref{lem:Factor_vs_SDP} to see that this is equivalent to the SDP
\begin{multline*}\inf\bigg\{a+ R ^{\top}a_1 + b+  C ^{\top}b_1 \ : \
 a_{1i}\geq 0  \text{ and }  a+a_{1i}\geq  A _{ii} \ \forall i,\\ b_{1j}\geq 0 \text{ and } b+b_{1j}\geq  B _{jj} \ \forall j,\left(\begin{array}{cc} A &X\\X^{\top}& B \end{array}\right)\succeq 0\bigg\}\;.\end{multline*}
\end{proof}

\begin{lemma}\label{lem:Factor_vs_SDP}
Let $f:\R^n\times \R^m\rightarrow \R$ be any function that is nondecreasing in each coordinate and let $X\in\R^{n\times m}$ be any matrix. Then
\begin{multline*}\inf\left\{f\left(\ttnorm{A_{(1)}},\dots,\ttnorm{A_{(n)}},\ttnorm{B_{(1)}},\dots,\ttnorm{B_{(m)}}\right) : AB^{\top}=X\right\}\\
= \inf\left\{f\left(\Phi_{11},\dots,\Phi_{nn},\Psi_{11},\dots,\Psi_{mm}\right):\left(\begin{array}{cc}\Phi&X\\X^{\top}&\Psi\end{array}\right)\succeq 0\right\}\;,\end{multline*}
where the factorization $AB^{\top}=X$ is assumed to be of arbitrary dimension, that is, $A\in\R^{n\times k}$ and $B\in\R^{m\times k}$ for arbitrary $k\in\mathbb{N}$.
\end{lemma}

\begin{proof}
We follow similar arguments as in Lemma 14 in \cite{SrebroThesis}, where this equality is shown for the special case of calculating a trace norm.

For convenience, we write
\[g(A,B)= f\left(\ttnorm{A_{(1)}},\dots,\ttnorm{A_{(n)}},\ttnorm{B_{(1)}},\dots,\ttnorm{B_{(m)}}\right)\]
and
\[h(\Phi,\Psi)=f\left(\Phi_{11},\dots,\Phi_{nn},\Psi_{11},\dots,\Psi_{mm}\right)\;.\]
Then we would like to show that
\[\inf\left\{g(A,B):AB^{\top}=X\right\}=\inf\left\{h(\Phi,\Psi):\left(\begin{array}{cc}\Phi&X\\X^{\top}&\Psi\end{array}\right)\succeq 0\right\}\;.\]

First, take any factorization $AB^{\top}=X$.
 Let $\Phi=AA^{\top}$ and $\Psi=BB^{\top}$. Then $\left(\begin{array}{cc}\Phi&X\\X^{\top}&\Psi\end{array}\right)\succeq 0$, and we have $g(A,B)=h(\Phi,\Psi)$ by definition. Therefore, 
\[\inf\left\{g(A,B):AB^{\top}=X\right\}\geq \inf\left\{h(\Phi,\Psi):\left(\begin{array}{cc}\Phi&X\\X^{\top}&\Psi\end{array}\right)\succeq 0\right\}\;.\]

Next, take any $\Phi$ and $\Psi$ such that $\left(\begin{array}{cc}\Phi&X\\X^{\top}&\Psi\end{array}\right)\succeq 0$. Take a Cholesky decomposition
\[\left(\begin{array}{cc}\Phi&X\\X^{\top}&\Psi\end{array}\right) = \left(\begin{array}{cc}A&0\\B&C\end{array}\right)\cdot \left(\begin{array}{cc}A&0\\B&C\end{array}\right)^{\top}= \left(\begin{array}{cc}AA^{\top}&AB^{\top}\\BA^{\top}&BB^{\top}+CC^{\top}\end{array}\right)\;.\]
From this, we see that $AB^{\top}=X$, that $\Phi_{ii}=\ttnorm{A_{(i)}}$ for all $i$, and that $\Psi_{jj}\geq\ttnorm{B_{(j)}}$ for all $j$. Since $f$ is nondecreasing in each coordinate, we have $h(\Phi,\Psi)\geq g(A,B)$.
Therefore, we see that
\[\inf\left\{g(A,B):AB^{\top}=X\right\}\leq  \inf\left\{h(\Phi,\Psi):\left(\begin{array}{cc}\Phi&X\\X^{\top}&\Psi\end{array}\right)\succeq 0\right\}\;.\]

\end{proof}

\end{document}